\documentclass[preprint]{article}

\PassOptionsToPackage{numbers, compress}{natbib}
\usepackage{neurips_2019}

\usepackage[utf8]{inputenc} % allow utf-8 input
\usepackage[T1]{fontenc}    % use 8-bit T1 fonts
\usepackage{hyperref}       % hyperlinks
\usepackage{url}            % simple URL typesetting
\usepackage{booktabs}       % professional-quality tables
\usepackage{amsfonts}       % blackboard math symbols
\usepackage{dsfont}
\usepackage{nicefrac}       % compact symbols for 1/2, etc.
\usepackage{microtype}      % microtypography
\usepackage{todonotes}
\usepackage{amsmath,amssymb}
\usepackage{amsthm}
\usepackage{amsfonts}

\newtheorem{theorem}{Theorem}

\newtheorem{lemma}[theorem]{Lemma}

\newcommand{\E}{\mathbb{E}}

\DeclareMathOperator*{\arginf}{arg\,inf}

\newcommand{\set}[2]{\left\{ #1 \,\middle|\, #2 \right\}}

\newcommand{\KL}[2]{D_{KL}\!\left(#1 \middle\| #2 \right)}

\newcommand{\ELBO}[2]{\text{ELBO}\!\left(#1, #2 \right)}

\newcommand{\OT}[3]{OT_{#3}\!\left(#1, #2 \right)}

\newcommand{\mean}[2]{\E_{#2} \! \left[ #1 \right]}

\newcommand{\de}[1]{\text{d}#1}

\title{Wasserstein variational gradient descent: From semi-discrete optimal transport to ensemble variational inference}

\author{
  Luca Ambrogioni \\
  Radboud University\\
  \texttt{l.ambrogioni@donders.ru.nl} \\
  \And
  Umut Güçlü \\
  Radboud University\\
  \texttt{u.guclu@donders.ru.nl} \
  \And
  Marcel A. J. van Gerven \\
  Radboud University\\
  \texttt{m.vangerven@donders.ru.nl} \\
}

\begin{document}

\maketitle
\begin{abstract}
Particle-based variational inference offers a flexible way of approximating complex posterior distributions with a set of particles. In this paper we introduce a new particle-based variational inference method based on the theory of semi-discrete optimal transport. Instead of minimizing the KL divergence between the posterior and the variational approximation, we minimize a semi-discrete optimal transport divergence. The solution of the resulting optimal transport problem provides both a particle approximation and a set of optimal transportation densities that map each particle to a segment of the posterior distribution. We approximate these transportation densities by minimizing the KL divergence between a truncated distribution and the optimal transport solution. The resulting algorithm can be interpreted as a form of ensemble variational inference where each particle is associated with a local variational approximation. 

\end{abstract}
\section{Introduction} 
Stochastic variational inference (VI) is becoming a cornerstone of modern machine learning research as it reduces Bayesian inference to a stochastic optimization problem that can be automatically solved using deep learning frameworks \cite{hoffman2013stochastic, ranganath2014black, rezende2014stochastic}. Particle-based VI methods have recently gained substantial popularity with the introduction of \emph{Stein variational gradient descent} (SVGD) \cite{liu2016stein}. In a particle-based variational method the posterior distribution is approximated as the stationary distribution of a system of interacting particles \cite{dai2015provable, liu2016stein}. In SVGD this dynamics can be decomposed into a steepest ascent term and a repulsive interaction between particles that avoids the collapse of the approximate posterior into its modes. While the dynamical system has been proven to converge to the exact posterior at the limit of infinitely many particles, the finite particle dynamics suffers from the fact that the repulsive force does not depend on the true posterior and is instead induced by an arbitrarily chosen kernel function \cite{liu2016stein}. This is particularly problematic when the dimensionality of the latent space is high since the performance of kernel methods rapidly degrade with the dimensionality. Intuitively, the repulsive effect should not be dependent on an arbitrary kernel. Instead, it should reflect an 'explaining away' phenomenon, where regions of the posterior that are already captured by a particle should not influence the dynamics of the other particles. 

Optimal transport theory is becoming another fundamental part of modern machine learning research~\cite{cuturi2013sinkhorn, arjovsky2017wasserstein, gulrajani2017improved, tolstikhin2017wasserstein}. Recently, optimal transport theory has been used to obtain a new flexible form of black-box stochastic variational Bayesian inference that replaces the usual Kullback-Leibler (KL) divergence with Wasserstein divergences \cite{ambrogioni2018wasserstein}. In this paper we introduce a particle-based form of Wasserstein variational inference based on the theory of semi-discrete optimal transport. We call this method Wasserstein variational gradient descent (WVGD). We approximate the posterior distribution by minimizing a semi-discrete optimal transport divergence. In this formulation mode collapse is avoided thanks to the explaining away phenomenon without resorting to repulsive forces. Importantly, the optimal transport formulation does not only provide the particle approximation but also a set of optimal transportation densities that map each particle to a continuous distribution that models a part of the posterior. Once we assume a parametric form, these transportation densities behave like local version of the parametric models used in conventional stochastic VI, each modeling the probability density in a region around the particle. Therefore, our approach can also be seen as a form of ensemble VI. 

\section{Related work} 
The use of optimal transport divergences in variational Bayesian inference problems was introduced in \cite{ambrogioni2018wasserstein}. However, Wasserstein variational inference is only applicable in case of joint-contrastive (amortized) inference problems, while WVGD can also be used without inference amortization. The WVGD method is a form of particle-based VI. Particle-based VI can be seen as an intermediate between sampling methods such as MCMC and conventional VI as it combines the non-parametric nature of samplers with an optimization point of view. The most popular particle-based VI algorithm is SVGD \cite{liu2016stein}. SVGD has been applied in several domains including reinforcement learning \cite{haarnoja2017reinforcement, liu2017stein} and meta learning \cite{kim2018bayesian, feng2017learning}. The theory behind our variational approach is closely related to optimal transport clustering \cite{laclau2017co, mi2018variational}. In these clustering approaches a series of medoids are trained to minimize the Wasserstein distance with a target distribution.

\section{Particle-based variational inference} 
In an approximate Bayesian inference problem, the aim is to approximate the posterior distribution of a latent variable $z$ given a set of observations $x$. The posterior distribution has the following form:
\begin{equation}
p(z|x) = \frac{p(x|z)p(z)}{p(x)}~.
\end{equation}
Unfortunately, the normalization constant $p(x)$ is usually intractable and has to be approximated. The basic idea behind VI is to approximate the real posterior with a more tractable family of parameterized densities $q_w(z)$ by minimizing a loss functional. The most commonly used functional is the Kullback–Leibler divergence:
\begin{equation}
\KL{q_w(z)}{p(z|x)} = \mean{\log{\frac{q_w(z)}{p(z|x)}}}{z \sim q(z)}~.
\end{equation}
Usually $q_w(z)$ has a parametric form that can range from a simple diagonal Gaussian to an highly complex distribution induced by a deep generative model \cite{rezende2015variational, kingma2016improved, dinh2016density}. In a (weighted) particle-based variational framework, the approximate distribution is a linear combination of delta functions:
\begin{equation}
q_N(z) = \sum_j^N \beta_j \delta(z - z^j)~,
\end{equation}
where $z^j$ is the coordinate of the $j$-th particle and $\beta_j$ is the weight associated with the particle. Ideally we would like to find the optimal set of particles by minimizing the divergence between $q_N$ and $p$. Unfortunately the KL divergence is not defined for distributions such as $q_N$ that are not absolutely continuous with respect to the Lebesgue measure. SVGD circumvents this problem by defining an interacting particle dynamics whose asymptotic distribution at the thermodynamic limit ($N \rightarrow \infty$) converges to the $p(z|x)$ under the topology induced by the KL divergence. The problem is well-posed since the infinite ensemble distribution is absolutely continuous \cite{NIPS2017_6904}. However, the SVGD dynamics is not necessarily optimal in practical applications when only a finite number of particles are used. Specifically, in SVGD individual particles cannot explain away any finite amount of probability mass since in the ideal asymptotic ensemble each particle only contributes infinitesimally to the overall posterior. Consequently, when $N$ is finite the proper coverage of the posterior depends on the choice of a kernel function that regulates the repulsive interactions between the particles. 

\section{Semi-discrete optimal transport} 
Optimal transport divergences measure the deviation between two distributions as the cost of optimally transporting a distribution to the other. An optimal transport divergence is defined by the following optimization problem:
\begin{equation}
\OT{q}{p}{c} = \inf_{\gamma(z,z') \in \Gamma[q,p]} \mean{c(z, z')}{z, z' \sim \gamma(z, z')}~,
\end{equation}
where $\Gamma[q,p]$ is the set of joint distributions having $q$ and $p$ as respectively the first and the second marginal. An important advantage of optimal transport divergences is that they can be used to compare discrete and continuous distributions. This form of optimal transport is called semi-discrete and can be formulated as follows \cite{peyre2017computational}:
\begin{align}
\OT{q_N}{p}{c} &= \inf_{\gamma \in \Gamma[q_N,p]} \mean{c(z, z')}{z, z' \sim \gamma(z, z')} \notag \\
&= \inf_{\zeta(z'|z) \in Z[q_N,p]} \sum_j^N \beta_j \mean{c(z^j, z')}{z' \sim \zeta(z'|z)}~,
\end{align}
where $Z[q,p]$ is the set of conditional distributions that fulfill the marginalization constraint: 
\begin{equation}
Z[q_N,p] = \set{\zeta(z'|z)}{\sum_j^N \beta_j \zeta(z'|z^j) = p(z')}~.
\end{equation}

\section{Particle-based inference with semi-discrete optimal transport} 
Our aim is to obtain a particle-based variational approximation by minimizing the optimal transport divergence between a weighted set of particles and the posterior distribution. Using semi-discrete optimal transport, we can formulate this optimal finite particle approximation of the posterior as the solution of the following joint optimization problem:
\begin{align}
(z_1^*,...,z_N^*, \beta_1^*,...,\beta_N^*) &=\arginf_{z_1,...,z_N} \arginf_{\beta_1,...,\beta_N} \OT{q_N}{p}{c}~ \notag \\
&= \arginf_{z_1,...,z_N} \arginf_{\beta_1,...,\beta_N} \left[ \inf_{\zeta(z'|z) \in Z[q_N,p]} \sum_j^N \beta_j \mean{c(z^j, z')}{z' \sim \zeta(z'|z)} \right]~ \notag~,
\end{align}
where the notion of optimality depends on the cost function $c$. The transportation densities $\zeta(z'|z)$ map each particle to a component of the posterior distribution. In other words, the transportation densities can be seen as emission models that spread the probability mass centered in a particle to its surroundings. 

\section{Formal solution of the optimal transport problem} 
The solution of the semi-discrete optimal transport problem has several interesting properties. The support sets of the transportation densities are elements of a tessellation of the $z$ space into non-overlapping sets. In the general case, these sets can be found using computational geometry algorithms \citep{aurenhammer1987power} and quasi-Newton solvers \cite{merigot2011multiscale}. Fortunately, the problem of finding these cells greatly simplify if we simultaneously optimize the transportation densities and the weights of the discrete distribution as stated in the following theorem:
\begin{theorem}[Formal solution of the optimal transport problem] \label{th: formal solution}
The optimization problem 
\begin{equation}
\arginf_{\beta_1,...,\beta_N} \left[ \inf_{\zeta(z'|z) \in Z[q_N,p]} \sum_j^N \beta_j \mean{c(z^j, z')}{z' \sim \zeta(z'|z^j)} \right]~
\end{equation}
is solved by the following tessellation:
\begin{equation}
L_j = \set{z'}{\forall{k}:~~ c(z^j, z') < c(z^k, z')}~,
\end{equation}
with optimal transportation densities obtained by restricting $p$ to each set of the tessellation:
\begin{equation}
p_j(z|x) = \frac{1}{\beta_j^*} p(z|x) \mathcal{I}_j(z)
\end{equation}
where $\mathcal{I}_j(z)$ is the indicator function of the set $L_j$. The optimal weights are given by the following expression:
\begin{equation}
\beta_j^* = \int_{L_j} p(z|x) \de{z}~.
\end{equation}
\end{theorem}
\begin{proof}
It is easy to see that transporting each point $z'$ to the particle $z^j$ such that $c(z^j,z')$ is the smallest leads to the smallest possible transportation cost. This implies that the $j$-th transportation density is supported on the set $L_j$ which (up to sets of zero measure) is disjoint from the supports of the other transportation densities. The marginalization constraint then imposes that, in the set $L_j$, $\gamma(z|z^j)$ is proportional to $p(z|x)$. In a general semi-discrete optimal transport problem this solution is not allowed since $\beta_j$ is not necessarily equal to $\int_{L_j} p(z|x) dz$, leading to a violation of the marginalization constraint. However, the solution is always possible in this joint-optimization since we can simply set $\beta_j$ to be equal to $\int_{L_j} p(z|x) dz$.

\end{proof}

\subsection{Deriving the gradient} 
The result at the end of the last section gives a closed-form solution for the particle-based variational loss:
\begin{align} \label{eq: formal loss}
\mathcal{L}(z^1,...,z^N) &= \OT{q_N}{p}{c} \notag\\
&= \sum_j \beta_j^* \int \frac{1}{\beta_j^*} p(z|x) \mathcal{I}_j(z) c(z^j, z) \de{z} \notag\\
&= \mean{\sum_j \mathcal{I}_j(z) c(z^j, z)}{z \sim p(z|x)}~.
\end{align}

In order to use this formula as the basis in a stochastic gradient descent algorithm we need to prove its differentiability. The loss in Eq.~\ref{eq: formal loss} depends on $z^j$ both directly through the cost and indirectly through the boundary of the indicator functions. The problem of differentiating functions of this form is well-known in Eulerian fluid dynamics \cite{leal2007advanced}. The following lemma is a direct consequence of the $n$-th dimensional generalization of the famous Reynolds transport theorem:
\begin{lemma}[Differentiability] \label{th: differentiability}
If the posterior density $p(z|x)$ is continuous and the cost function $c(z_1,z_2)$ is differentiable with continuous partial derivatives, the loss function in Eq.~\ref{eq: formal loss} is differentiable.
\end{lemma}
\begin{proof}
Using Reynolds transport theorem, we can write the partial derivative with respect to the $h$-th component of the $k$-th particle as follows:
\begin{equation} \label{eq: partial derivative}
\frac{\partial \mathcal{L}}{\partial z^k_{h}} = \sum_j \left[ \int_{L_j} \frac{\partial c}{\partial z^k_{h}}\!(z^j,z) p(z|x) \de{z}  + \int_{\partial L_j} c(z^j,z) \left( v_j(z) \cdot d\Sigma_j(z) \right) \right]~.
\end{equation}
The second integral in this expression is defined over the frontier $\partial L_j$ of $L_j$. The differential $v_j(z) \cdot d\Sigma_j(z)$ gives the component of the velocity of the boundary $v_j(z)$ that is parallel to the surface normal vector $d\Sigma_j(z)$. The partial derivative is clearly continuous since it is a sum of integrals of continuous functions. The statement follows from the differentiability theorem. 
\end{proof}
The partial derivatives in Eq.~\ref{eq: partial derivative} suggest that the gradient descent dynamics of the particles is driven by two terms. The first term drives the $j$-th particle towards the centroid of its own set $L_j$. Under this dynamics, the particles interact only by explaining away parts of the posterior, thereby screening the other particles from the attractive forces contained in their own set. The second term is defined at the frontiers of the sets and seems to imply the existence of more direct interactions. However, this term is actually zero as stated in the following theorem:

\begin{theorem}[Gradient] \label{th: gradient}
The gradient of Eq.~\ref{eq: formal loss} with respect to the position of the $j$-th particle is given by the following expression:
\begin{equation} \label{eq: formal gradient}
    \nabla_j \mathcal{L}(z^1,...,z^N) = \mean{\nabla_j c(z^j, z)}{z \sim p_j(z|x)}~.
\end{equation}
\end{theorem}
\begin{proof}
The partial derivative in Eq.~\ref{eq: partial derivative} can be rewritten as follows:
\begin{equation} 
\frac{\partial \mathcal{L}}{\partial z^k_{h}} = \int_{L_j} \frac{\partial c}{\partial z^k_{h}}\!(z^k,z) p(z|x) \de{z}  + \sum_j \int_{\partial L_j} c(z^j,z) \left( v_j(z) \cdot d\Sigma_j(z) \right)~.
\end{equation}
The first integral in this expression gives the gradient in the statement. Consequently, we need to show that the term
\begin{equation} \label{eq: boundary term}
\sum_j \int_{\partial L_j} c(z^j,z) \left( v_j(z) \cdot d\Sigma_j(z) \right)~
\end{equation}
is equal to zero. It is easy to see that each boundary $\partial L_j$ is the intersection of two-particles boundaries:
$$
\partial L_j = \bigcap_h \mathcal{B}_{jh}~,
$$
where $\mathcal{B}_{jh} = \set{z}{c(z_j,z) = c(z_k,z)}$. Consequently, Eq.~\ref{eq: boundary term} can be decomposed into a sum of integrals over subsets of these two-particles boundaries:
$$
\sum_{j,h \neq j} \left[ \int_{A_{j,k} \subset \mathcal{B}_{jh}} c(z^j,z) \left( v_j(z) \cdot d\Sigma_j(z)\right) + \int_{A_{j,k} \subset \mathcal{B}_{jh}} c(z^k,z) \left( v_k(z) \cdot d\Sigma_k(z) \right) \right]~,
$$
where the two terms of this expression are the two integrals at the opposite side of a boundary. The velocities $v_j(z)$ and $v_k(z)$ are equal since they describe the motion of the same boundary. The two costs $c(z^j,z)$ and $c(z^k,z)$ are also equal on the boundary by definition. Finally, the unit length normal elements $d\Sigma_j(z)$ and $d\Sigma_k(z)$ have opposite direction as a set ends where the other begins. Consequently, the two integrals have equal magnitude and opposite sign and the final result is equal to zero.
 
\end{proof}

\section{Wasserstein variational gradient descent} 
We can now introduce the WVGD algorithm. The gradient descent dynamics of the system of $N$ particles is given by the following system of differential equations:
\begin{equation}
\dot{z}^j(t) = -\mean{\nabla_j c(z^j, z)}{z \sim p_j(z|x)}~.
\end{equation}
It is instructive to study some special cases. If we use a single particle, the loss function becomes:
\begin{equation}
\mathcal{L}(z^1) = \mean{c(z^1, z)}{z \sim p(z|x)}~.
\end{equation}
The minimum of this loss is the $c$-medoid of the posterior distribution \cite{park2006k}. When $c$ is the squared Euclidean distance, the one-particle dynamics converges to the posterior mean. This behavior differs from SVGD, where the one-particle case reduces to the gradient ascent of the log posterior \cite{liu2016stein}. In the multi-particle case particles interact by screening part of the posterior density by moving the boundaries of the sets $L_j$. Consider the two-particles case in one dimension with $z^1 < z^2$. The velocity field of particle one can be decomposed into two terms:
\begin{equation}
\dot{z}^1(t) = \int_{-\infty}^{\infty} c(z^1, z) dp(z|x) - r_{1,2}~.
\end{equation}
The first term is the global attractive effect of the posterior density on particle one while the second term is an apparent repulsive interaction between particle one and particle two:
\begin{equation}
r_{1,2} = \int_{(z^1 + z^2)/2}^{\infty} c(z^1, z) dp(z|x)~.
\end{equation}
This repulsive interaction is a function of the posterior mass contained in the half-axis covered by the second particle. This is in stark contrast with SVGD where the repulsive interactions are only a function of the location of the particles. 

\section{Adaptive importance sampling and ensemble variational inference} 
Since we cannot directly sample from the posterior, to obtain an Monte Carlo estimate of the gradient we resort to importance sampling. The resulting (discretized) stochastic dynamics has the following form:
\begin{align}
z^k(t + 1) &= z^k(t) - \frac{\lambda}{A} \sum_m \alpha_m \left[\nabla_k c(z^k(t), \zeta^m) \right], \notag\\
\zeta^m &\sim q_j(\theta_j(t))~.
\end{align}
where $\lambda$ is the learning rate, $q_j(z_m; \theta_j(t))$ is the importance sampling distribution of the $j$-th particle parameterized by $\theta_j(t)$, $\alpha_m = p_j(\zeta^m|z^j)/q_j(\zeta^m; \theta_j)$ is an importance weight and $A = \sum_m \alpha(z)$ is a normalization term. The variance of the gradient can vary greatly depending on the choice of the sampling distribution and, since the distributions $p_j(z|x)$ shift during the time development of the system, it is important to update the sampling distributions together with the positions of the particles \cite{karamchandani1989adaptive}. A flexible choice for the adaptive importance sampling dynamics is to descend the gradient of the reverse KL divergence between the parameterized sampling distributions and the transportation densities:
\begin{equation}
\theta_j(t + 1) = \theta_j(t) - \epsilon \nabla_{\theta_j} \KL{q_j(z; \theta_j(t))}{p_j(z|x)}~.
\end{equation}
An interesting side effect of this choice is that we simultaneously obtain the positions of the particles and tractable approximations for the transportation densities. Therefore, the resulting algorithm provides both the $N$-particles approximation and an associated continuous approximation that can be seen as a form of ensemble VI. In the one-particle case the algorithm reduces to conventional VI.

We parameterize the sampling distribution $q_j$ as the restriction of a tractable distribution $q$ to the set $L_j$:
\begin{equation}
    q_j(z;\theta_j) = \frac{1}{Z_j(\theta_j)} \mathcal{I}_j(z) q(z; \theta_j)~,
\end{equation}
where
\begin{equation}
    Z_j(\theta_j) = \int_{L_j} q(z; \theta_j) \de{z}~.
\end{equation}
For example, if $q$ is a multivariate Gaussian, the sampling distribution will be a truncated Gaussian. Note that this restriction is necessary since the KL divergence can only be defined between distributions that share the same support set and the optimal transportation density is supported on $L_j$. The gradient of the KL divergence is given by:
\begin{align}
\nabla_{\theta_j} \KL{q_j(z)}{p_j(z|x)} = - \nabla_{\theta_j} \mean{\log{p(z,x)}}{z \sim q_j} + \nabla_{\theta_j} \mathcal{S}[q_j]~,
\end{align}
where $\mathcal{S}[q_j]$ is the entropy of $q_j$. 
An unbiased estimate of the gradient of an expectation with respect to $q_j$ can be obtained using the standard reparameterization trick \cite{rezende2014stochastic} since the set $L_j$ does not depend on $\theta_j$. The only difference with regular evidence lower bound (ELBO) maximization is that samples are rejected if they fall outside $L_j$. A problem in computing the gradient of the entropy is that we do not have a closed-form expression for the normalization constant $Z_j(\theta_j)$. However, an expression for the gradient of the entropy is given by the following theorem:
\begin{theorem}[Gradient of the entropy] \label{th: gradient entropy}
The gradient of $\mathcal{S}[q_j]$ with respect to the position of the parameters $\theta_j$ is given by the following expression:
\begin{align}
    \nabla_{\theta_j} \mathcal{S}[q_j] &= \nabla_{\theta_j} \mean{\log{q(z;\theta_j)}}{z \sim q_j} - \nabla_{\theta_j} \log{Z_j(\theta_j)} \notag \\
    &= \nabla_{\theta_j} \mean{\log{q}(z;\theta_j)}{z \sim q_j} - \mean{\nabla_{\theta_j}\log{q(z;\theta_j)}}{z \sim q_j}~.
\end{align}
\end{theorem}
\begin{proof}
\begin{align}
    \nabla_{\theta_j} \log{Z_j(\theta_j)} &= \frac{1}{Z_j} \nabla_{\theta_j} \int_{L_j} q(z; \theta_j) dz \notag \\ 
    &= \int_{L_j} \left(\nabla_{\theta_j}\log{q(z;\theta_j)} \right) \frac{q(z;\theta_j)}{Z_j} dz \notag \\
    &= \mean{\nabla_{\theta_j}\log{q(z;\theta_j)}}{z \sim q_j}
\end{align}
\end{proof}

\subsection{Partitioned evidence lower bound} 
Using the ensemble of transportation densities, we can obtain a lower bound of the model evidence. Consider the partition $\Pi$ induced by the sets $L_j$. We define the partitioned ELBO (PELBO) as follows:
\begin{align}
    \text{PELBO}\left[p,q, \Pi \right] = \sum_j \beta_j^* \ELBO{q_j}{p_j}~,
\end{align}
where $\beta^*_j$ is the optimal weight defined in Theorem~\ref{th: formal solution}. Using the Jensen's inequality, it is easy to show that the PELBO is an evidence lower bound: 
\begin{align}
    \sum_j \beta_j^* \ELBO{q_j}{p_j}  \leq  \sum_j \beta_j^* \log{p_j(x)} \leq \log{\!\left( \sum_j \beta_j^* p_j(x) \right)} = \log{p(x)}~,
\end{align}
where $p_j(x)$ is given by $\int p_j(z,x) \de{z}$. Note that the weights $\beta_j^*$ are not available in closed-form and need to be approximated by importance sampling.

\section{Interpretability of the particle representation}
An interesting feature of WVGD is that the partition of the parameter space into regions covered by individual particles often leads to a semantically interpretable partition of the posterior distribution. An interesting example is given by the posterior distribution of the covariance hyper-parameters of a Gaussian process regression. Consider the following quasi-periodic covariance function:
\begin{equation}
k(t,t') = A e^{-(t - t')/{2 s^2}} \cos{\left( 2 \pi f (t - t') \right)} + B \delta(t - t')
 \label{eq: GP covariance function}
\end{equation}
where $A$ is the amplitude of the oscillatory component, $f$ is the oscillatory frequency, $s$ is the smoothness of the oscillatory envelope and $B$ is the amplitude of the white noise component. Different values of these hyper-parameters corresponds to radically different functional behavior. If the period $1/f$ is small compared with the length scale $s$, then the resulting function is smooth but not oscillatory. Conversely, oscillations emerge when $1/f$ if smaller than the envelope length scale. This can lead to two different "interpretations" of the data. In one "interpretation" the data is not oscillatory and the fluctuations are explained by the white noise. In the other the ampliotude of the noise is low and the fluctuations are exmplained by the oscillatory nature of the process. These "interpretations" are reflected by different modes of the posterior distribution over the hyper-parameters of the covariance function. In WVGD, idividual particles can model these semantically meaningful modes leading to an interpretable partition of the parameter space. This behavior is shown on figure \ref{fig: GP results}.

\begin{figure}[ht]
    \centering
    \includegraphics[width=0.7\textwidth]{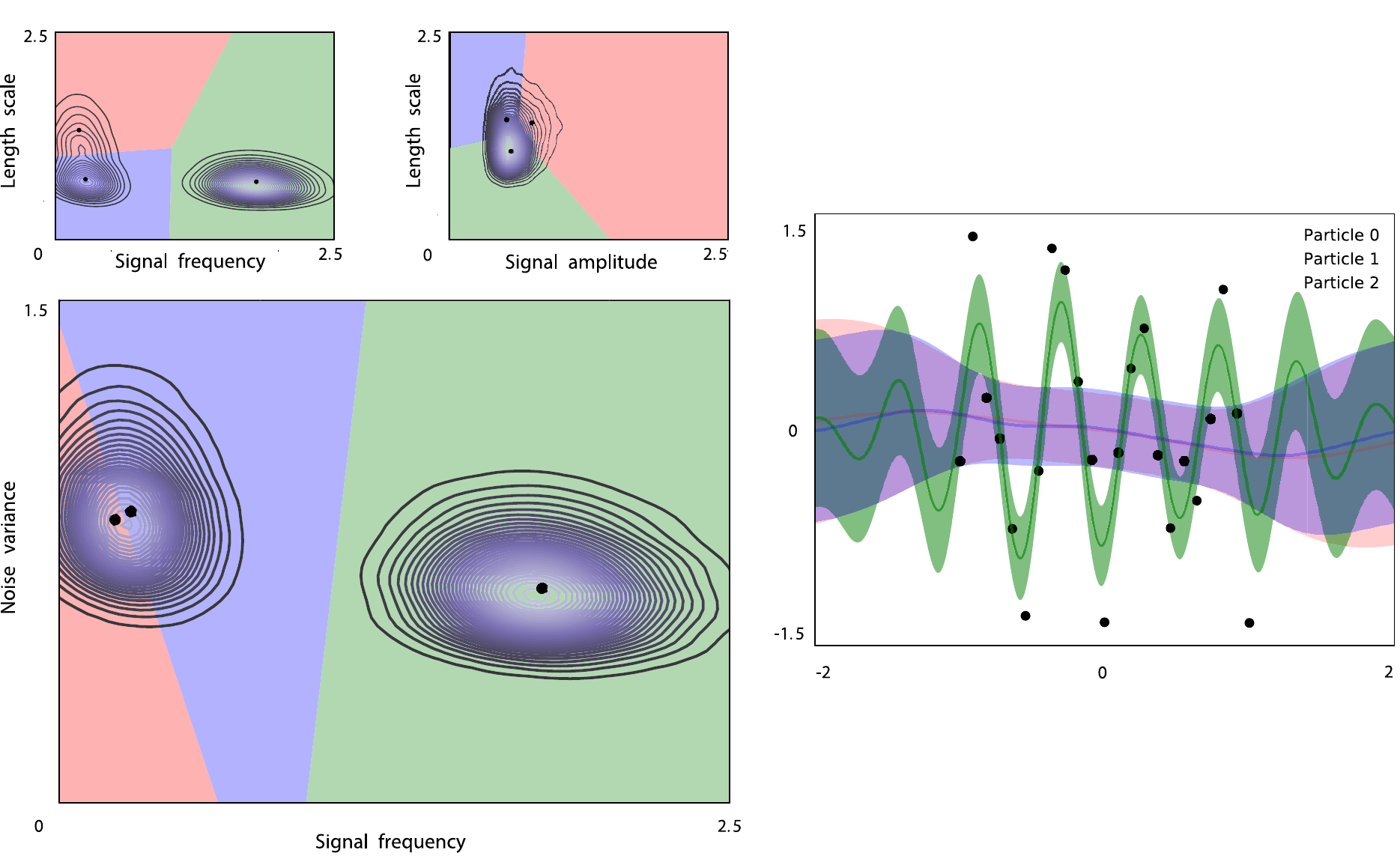}
    \caption{Partition of the posterior distribution of the covariance hyper parameters of a Gaussian process. The color denotes the region of the parameter space associated with each particle.}
    \label{fig: GP results}
\end{figure}

%\section{ensemble of Bayesian networks}
%$c(W_1,W_2) = \text{arcos}\!\left({\frac{\mean{f(x;W_1)f(x;W_1)}{}}{\sqrt{\mean{f(x;W_2)^2}{}\mean{f(x;W_1)^2}{}}}}\right)$

%TODO: Theorem of universal approximation

\section{Experiments} 
\subsection{Mixture of Gaussians}

\begin{figure}[ht]
    \centering
    \includegraphics[width=0.8\textwidth]{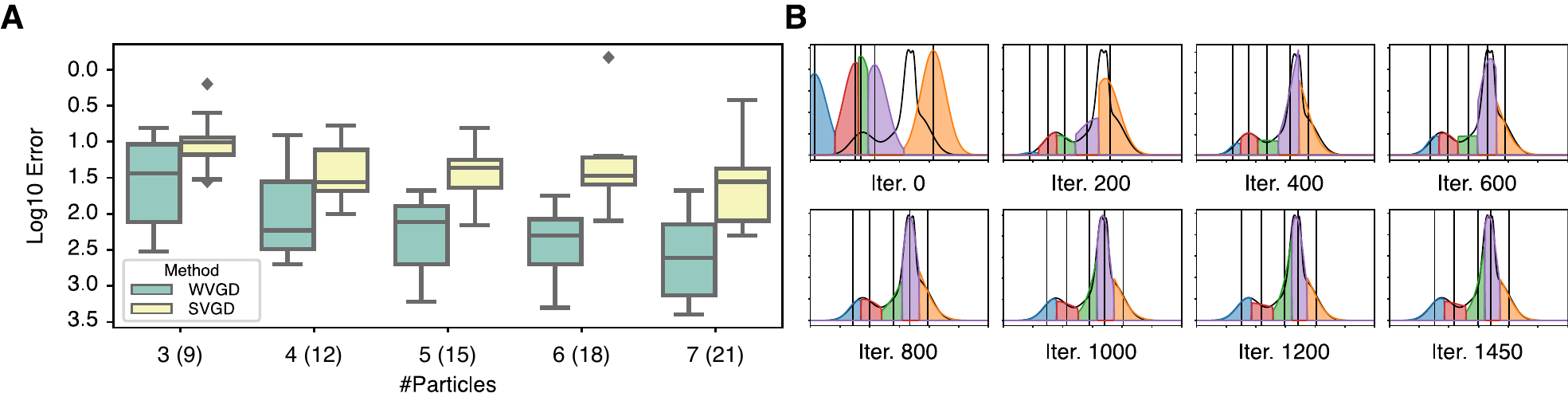}
    \caption{A) Squared error of WVGD and SVGD as a function of the number of particles. SVGD was trained with three times as many particles as WVGD. B) Dynamics of WVGD particles and transportation densities during training.}
    \label{fig: results}
\end{figure}

As a first step, we validate our WVGD method on a simple mixture of Gaussians problem. Random target posterior distributions were obtained by sampling the mixture ($\sim \mathcal{U}(0, 1)$), location ($\sim N(0,1.5)$) and scale parameters ($\sim \mathcal{U}(0.1, 1))$ of a mixture of five Gaussian distributions. We quantified the performance of WVGD and SVGD using the squared error between posterior density and variational approximation. The cost function in WVGD was the squared euclidean distance. We used a a Gaussian kernel for SVGD. The bandwidth was adapted during training using the formula $bw = \text{md}^2/\log{n}$,  where $\text{md}$ is the median of the distances between the particles as suggested in \cite{liu2016stein}. The SVGD density was obtained by using kernel density estimation while the WVGD density was given by its transportation densities. In order to have a fair comparison, we fit the bandwidth parameter of the kernel density estimation individually on each groundtruth posterior, this leads to a small bias in favor of SVGD. Furthermore, since the WVGD method has three times more parameters than SVGD, we compared the error of WVGD with the error of SVGD with three times as many particles. The experiment was repeated $15$ times for each number of particles ($3(9), 4(12), 5(15), 6(18)$ and $7(21)$) with a randomly generated groundtruth density. The histograms of the errors are shown in Fig.~\ref{fig: results}. WVGD has consistently higher performance than SVGD.

\subsection{Logistic regression}
We tested the performance of the method as function of the number of particles in a Bayesian logistic regression experiment on real data. We used the iris, Boston house pricing and breast cancer and diabetes datasets~\citep{harrison1978hedonic, efron2004least, street1993nuclear}. Continuous targets were binarized based on being bigger or smaller than the median. The weights of our models had standard Gaussian prior distributions and the variational posterior assigned to each particle was a fully factorized Gaussian over the weights. Note that the one-particle case is our baseline as it corresponds to standard stochastic VI. We quantify the performance as the ELBO (PELBO).  We did not include a comparison with SVGD as it is not straightforward to obtain an estimate of the ELBO using this method. Each experiment was repeated $20$ times on a random subsectection of $n = 50$ datapoints of the training set. The means of the ELBOs as function of the number of particles are given in the table, where the error is quantified as the standard error of the mean. As expected, the ELBO tend to increase as function of the number of particles. Forthermore, the multi-particle analysis always perform better than standard VI.

\begin{table}[ht]
\begin{center} 
 \begin{tabular}{l| l l l l l} 
  &  VI (1p) & 2p & 3p & 4p & 5p\\ [0.5ex] 
\hline
%Boston &  $-245.43 \pm 32.77$ & $-222.24 \pm 19.00$ & $-198.48 \pm 24.36$ & $-172.71 \pm 13.73$ & $-160.29 \pm 5.55$   \\ 
Boston &  $-245 \pm 32$ & $-222 \pm 19$ & $-198 \pm 24$ & $-172 \pm 13$ & $\bf{-160 \pm 5}$   \\ 
%Diabetes &  $-27.09 \pm 0.39$ & $-25.97 \pm 0.41$ & $-25.32 \pm 0.48$ & $-25.64 \pm 0.34$ & $-25.50 \pm 0.42$   \\ 
Diabetes &  $-27.1 \pm 0.4$ & $-26.1 \pm 0.4$ & $\bf{-25.3 \pm 0.5}$ & $-25.6 \pm 0.3$ & $-25.5 \pm 0.4$   \\ 
% \hline
 Iris  & $-3.41 \pm 0.27$ & $-3.15 \pm 0.28$ & $-2.29 \pm 0.19$ & $\bf{-2.11 \pm 0.21}$ & $-2.24 \pm 0.17$ \\ 
 %\hline
 %Breast cancer & $-535.036 \pm 78.72$ & $-255.30 \pm 22.33$ & $-252.99 \pm 21.79$ & $-217.07 \pm 14.45$ & $-205.78 \pm 15.39$ 
  Cancer & $-535 \pm 78$ & $-255 \pm 22$ & $-252 \pm 21$ & $-217 \pm 14$ & $\bf{-205 \pm 15}$  
  
\end{tabular} 
 \label{table1}
\end{center}
\end{table}

%\section{Conclusion} 
%In this paper we introduced WVGD as a new form of particle-based variational inference based on optimal transport theory. An interesting feature of our method is that it also provides an ensemble of truncated continuous densities that are variational approximations of the optimal transportation densities. In our experiments we showed that WVGD has higher performance than SVGD in a mixture of Gaussians problem.

%\section{Experiments} 
%\subsection{Mixture of Gaussians} 
%\subsection{Logistic regression} 
%\subsection{Mixture of Gaussian processes} 

\bibliographystyle{unsrtnat}
\bibliography{reference}

\end{document}